%% file: ijcai17.tex
\documentclass{article}
\usepackage{ijcai17}

\usepackage{times}

\usepackage{graphicx,amssymb,amstext,amsmath, amsthm}
\usepackage{color}
\usepackage{algorithm}
\usepackage{algpseudocode}
\usepackage{framed}
\usepackage{subfig}

\newcounter{tempcounter}

\newcommand{\citep}[1]{\citeauthor{#1} (\citeyear{#1})}
\newcommand{\bs}[1]{\boldsymbol{#1}}

\newcommand{\Es}{\mathcal{E}}
\newcommand{\Ex}{\mathbb{E}}

\newtheorem{thm}{Theorem}

\newcommand{\B}[1]{\textbf{#1}}

\newcommand{\Bu}{\mathcal{B}}
\newcommand{\Index}{\mathcal{I}}

\DeclareMathOperator{\I}{I}

\DeclareMathOperator{\Prob}{Pr}

\title{XOR-Sampling for Network Design with Correlated Stochastic Events}
\author{Xiaojian Wu, }

\author{Xiaojian Wu$^*$ ~~~~ Yexiang Xue\thanks{equal contribution.}  ~~~~ Bart Selman ~~~~ Carla P. Gomes  \\[4pt]
    Department of Computer Science, Cornell University, Ithaca, NY \\ [4pt] 
  {\tt xiaojian@cornell.edu} ~~~~ {\tt \{yexiang,selman,gomes\}@cs.cornell.edu }
  }

\begin{document}

\maketitle

\begin{abstract}
Many network optimization problems can be formulated as stochastic network design problems in which edges are present or absent stochastically. Furthermore, protective actions can guarantee that edges will remain present.  We consider the problem of finding the optimal protection strategy under a budget limit in order to maximize some connectivity measurements of the network. Previous approaches rely on the assumption that edges are independent. 
In this paper, we consider a more realistic setting where multiple edges are not independent due to natural disasters or regional events that make the states of multiple edges stochastically correlated. 
We use Markov Random Fields to model the correlation and define a new stochastic network design framework. We provide a novel algorithm based on Sample Average Approximation (SAA) coupled with a Gibbs or XOR sampler. The experimental results on real road network data show that the policies produced by SAA with the XOR sampler have higher quality and lower variance compared to SAA with Gibbs sampler.
\end{abstract}

\section{Introduction}

Many problems such as influence maximization~\cite{Kempe03}, river
network design~\cite{neeson2015enhancing} and pre-disaster
preparation~\cite{peeta2010pre} can be formulated as stochastic
network design problems, in which the edges of the network are present
or absent stochastically.
The goal is to take protection actions, which reinforce the
connectivity of a subset of edges with a limited budget, to maximize
the connectivity of the network under stochastic events.
For example, in the road network design problem, as shown in
Figure~\ref{fig:road_network}, edges can be present or absent from
the network stochastically when a natural disaster strikes. The
problem then is to invest smartly under a limited budget to strengthen
a few key road segments, so that ambulance centers can still remain
connected to population centers under the attacks of multiple natural
disasters.

Despite the worst-case $\mbox{PP}$-hard
complexity to evaluate the connectivity values under a fixed protection plan~\cite{Dyer2006StochasticProgramming}, approximate algorithms based on
Sample Average Approximation (SAA) have been developed to find nearly
optimal solutions~\cite{Sheldon10,Wu2015}, which convert the original
stochastic optimization problem into a deterministic one, by
optimizing the connectivity objectives over a fixed set of sample
scenarios.

The success of previous approaches depends on the core assumption that
the state of each edge (e.g., present or absent) is independent of
those of other edges. Under this assumption, high-quality samples can
be drawn with simple schemes such as flipping a biased coin for each
edge. Also, it has been  shown that often a small number of
samples (e.g., 10) is sufficient to find a good
policy~\cite{Sheldon10}, despite the fact that a much larger number of samples is
needed for theoretical guarantees~\cite{SAA02}.

Nevertheless, the states of edges are often correlated in
practice. For example, in the road network design problem shown in
Figure~\ref{fig:road_network}, edges tend to fail concurrently when a
regional event (e.g., a natural disaster such as an earthquake or a
snow storm) disables the connectivity of all edges in a given area.
In this paper, we study this more realistic setting where the states
of multiple edges can be correlated.  We propose to use a general
probabilistic model -- Markov Random Field (MRF) -- to model the joint
distribution of edge states, which can be built with domain knowledge
or learned from data.  

We make several contributions: {\bf(1)} We
propose a new stochastic network design problem framework, taking into
account correlated stochastic events based on MRF; {\bf(2)} We extend
the SAA framework with a novel sampling scheme to convert the original
stochastic network design problem into a deterministic network design
problem, which is formulated as a Mixed Integer Program (MIP) and
solved by CPLEX 12.6.  As we will show, also mentioned in
paper~\cite{Kumar12}, the MIP can only scale to a small number of
samples (e.g., $30$ samples).  Consequently, a \emph{small} but
\emph{good quality} set of samples is critical to produce high-quality
solutions; {\bf(3)} Another  key contribution of this
work is the  coupling of 
XOR-sampling~\cite{Gomes2006Sampling,ermon2013embed} with SAA;
{\bf(4)} We show that XOR-sampling coupled with SAA outperforms the
standard Gibbs sampler when coupled with SAA.  As shown in
paper~\cite{ermon2013embed}, the XOR sampler can produce MRF samples
with better quality but often has worse scalability than the Gibbs
sampler.
However, in our setting, the bottleneck of the runtime is in solving
the MIP, not the time required by the XOR sampler. This is because the
MIP solver can only handle relatively small samples (therefore they
have to be carefully selected). Given the relatively small sample
sizes, the time required by the XOR sampler to produce such samples is
negligible. {\bf(5)} We test our algorithms on the flood preparation
problem with real-world road networks~\cite{xiaojian2016}.  The goal
is to maintain the connectivity of road networks for emergency medical
services (EMS).  The results show that SAA with XOR sampling produces
policies that are never inferior in quality compared to SAA with
Gibbs. In other words, the policies produced by the SAA with XOR
sampling consistently have similar or better quality than SAA with
Gibbs sampling. Furthermore, in some cases, a small number of samples
(e.g., $10$) from the XOR sampler, is  enough for the SAA to
converge. Also, the solution quality with XOR sampling is 
consistently good, not very sensitive to the variance of samples used in
SAA.  In comparison, the solution quality  produced by a Gibbs
sampler has a higher variance when the number of samples used in SAA
is small.  These observations suggest that, on large-scale networks,
the SAA method can still produce a good policy if we use high-quality
samples such as those from XOR sampling.

\section{Problem Statement}
One important application of the stochastic network design problem is road network design~\cite{peeta2010pre}. 
Natural disasters such as earthquakes and floods break road segments and destroy the connectivity of the whole road network.
The paper~\cite{xiaojian2016} introduces the problem of flood preparation for EMS as shown in Fig.\,\ref{fig:road_network}.
In this problem, during floods, roads may be washed out and paths connecting ambulance centers to patients are cut off.
\begin{figure}[t]
\centering
\includegraphics[height=40mm]{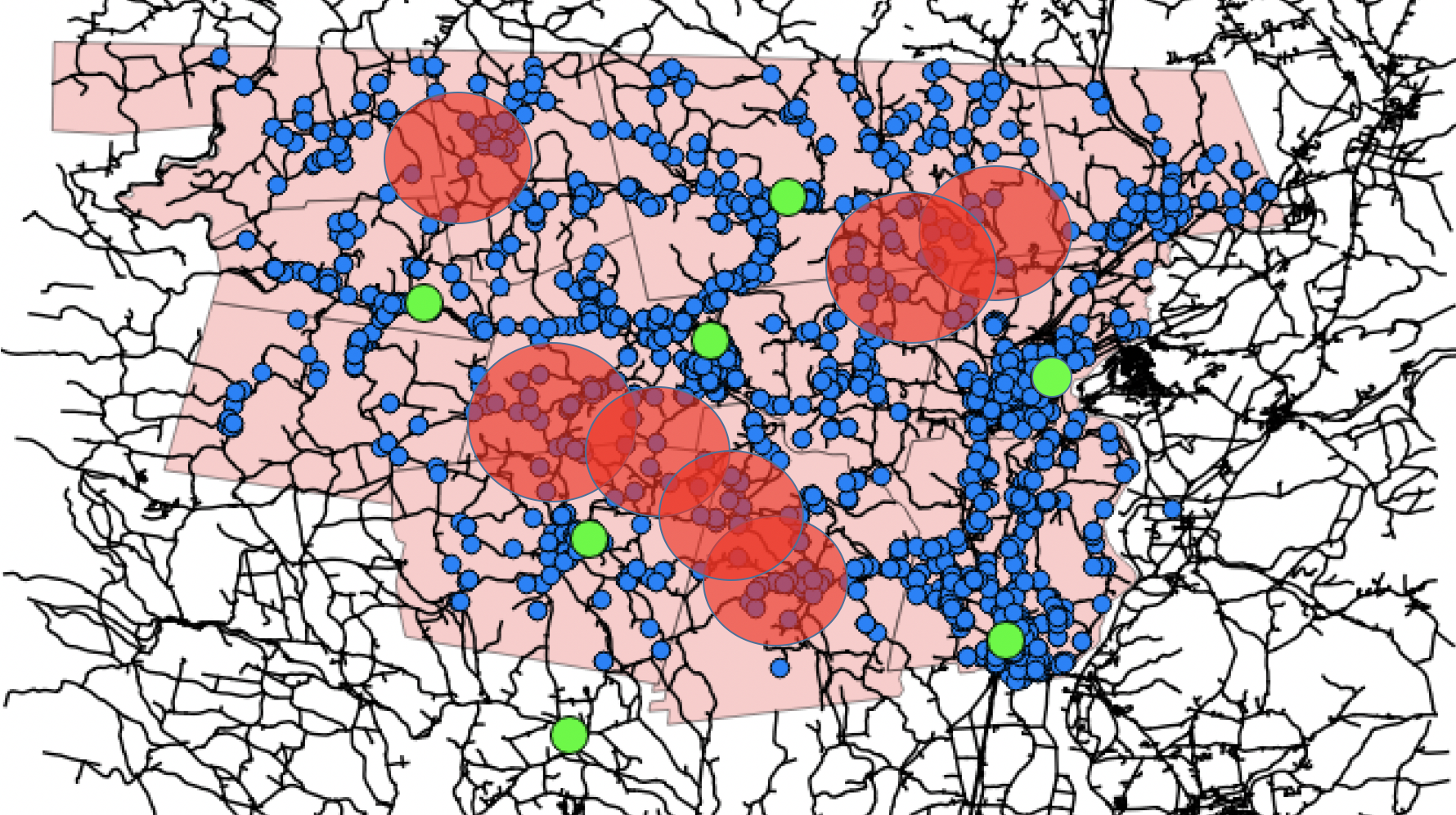}
\caption{In an exemplified road network design problem, the goal is to maintain connectivity between ambulance centers (green dots) and population centers (blue dots) in case of natural disasters. The stochasticity of multiple edges is correlated. For example, one natural disaster such as earthquake may affect the connectivity of all edges in its range (shown in red circles).}
\label{fig:road_network}
\end{figure}
To maintain good connectivity of the road network in case of an emergency, spatially critical roads should be identified and protected beforehand, at certain costs.
The question is how to find these spatially critical roads, given a budget limit, such that most of the patients can be reached when floods happen.
This problem is an example of the stochastic network design problem.

A stochastic network design problem is defined on a directed graph $G=\{V, E\}$ with a source node $s$.
Each node is associated with a weight $w(v)$ (e.g., population density).
Each edge $e\in E$ is associated with a binary random variable $\bs{\theta}_e$ that describes the state of the edge during disasters.
$\bs{\theta}_e = 1$ means that the edge $e$ is \emph{present} or not destroyed.
$\bs{\theta}_e = 0$ means that $e$ is \emph{absent} or destroyed. 
Let $\bs{\theta} = \{\bs{\theta}_e, e\in E\}$ denote the vector of random variables of all edges, each random variable taking value in $\{0,1\}^{|E|}$ and $\Prob(\bs{\theta})$ denote their joint probability distribution.
$\Prob(\bs{\theta})$ may be constructed from domain knowledge or learned from real-world data.
Note that we use the bold symbol $\bs{\theta}$ for random variables and the unbold symbol $\theta$ for their values.
For a given $\theta$, the indicator function $\I(s \leadsto v | \theta)$ indicates whether node $v$ is reachable from the source $s$ in the network by a path of all present edges. 
If node $v$ is reachable, it equals to $1$. Otherwise, it equals to $0$.
The expected weighted number of nodes that are reachable from the source $s$ is written as
\begin{align}
{\small
\Ex_{\bs{\theta}}\left[\sum_{v\in V} \I(s \leadsto v | \bs{\theta}) w(v)\right] = \sum_{\theta \in \{0,1\}^{|E|}}\Prob(\theta)\sum_{v\in V} \I(s \leadsto v | \theta) w(v)}
\label{eq:objective}
\end{align}

To be general, one protection action can protect one or multiple edges.
If an edge is protected, it is always present. 
Formally, we are given a collection of $M$ edge sets $\Es = \{E_1, E_2,..., E_M\}$ where each $E_s \subseteq E$ represents an action to protect all edges in $E_s$ and is associated with a cost $c_s$.
A policy $\pi$ is a subset of $\Es$ that indicates all edges being protected. 
Protecting an edge from $u$ to $v$ can be equivalently viewed as adding an edge $(u,v)$ that is always present in the network.
That is, if $(u, v)$ is protected by $\pi$, $u$ is connected to $v$ no matter $\theta_{(u,v)} = 1$ or $0$, which makes the indicator function depend on $\pi$ as $\I(s \leadsto v | \theta, \pi)$.
In other words, protection actions only change the network structure but not $\Prob(\bs{\theta})$, the probability of the stochastic events such as natural disasters that is independent of the network structure.
The decision making problem is to find the policy $\pi$, subject to a budget limit $\Bu$, to maximize the expected weighted number of nodes that are reachable from $s$, that is
\begin{align}
{\small
\max_{\pi \subseteq \Es} \Ex_{\bs{\theta}}\left[\sum_{v\in V} \I(s \leadsto v | \bs{\theta}, \pi) w(v)\right] ~~~~s.t. ~~~\sum_{E_s\in \pi} c_s \leq \Bu
}
\label{eq:max_exp}
\end{align}

This problem definition is general and subsumes many special cases. For example, in the real-world problem in our experiments, only a subset of edges are present or absent stochastically while the remaining edges are known to be present. For this special case, we can set the protection costs of those edges that are always present to zero so they are protected automatically in the optimal solution.

\vspace{5pt}
\noindent \B{Multiple Source Case~} In problem~(\ref{eq:max_exp}), it is assumed that there is only one source. 
It can be extended to multiple sources.
For example, there are several ambulance centers, each of which is represented by a source. The house of a patient is encouraged to be reachable from multiple centers in case that ambulences in one center are busy and can't respond.
In this new problem formulation, instead of having a single source $s$, we are given a set of sources $S$.
Connecting to more sources increases the objective value accordingly.
For each source in $S$, we count the weighted number of nodes that are reachable from the source separately.
The objective is to maximize the expected total weighted number of nodes that are reachable from each source in $S$, that is, 
\begin{align}
{\small \max_{\pi \subseteq \Es} \sum_{s\in S} \Ex_{\bs{\theta}}\left[\sum_{v\in V} \I(s \leadsto v | \bs{\theta}, \pi) w(v)\right] ~~~~s.t. ~~~\sum_{E_s\in \pi} c_s \leq \Bu }
\label{eq:multi_source}
\end{align}
In this case, each node in $S$ issues a separate network flow to each other node.
Later, we will introduce the algorithm to solve single source network design problems and it can be extended to solve multiple source problems.

\vspace{5pt}
\noindent\B{Model $\Prob(\bs{\theta})$~} We propose to use the Markov random field (MRF) to model the joint probability distribution $\Prob(\bs{\theta})$ of all variables.
MRF is a general model for the joint distribution of multiple correlated random variables.
Using notions in MRF, we define the probability of edge states $\Prob(\bs{\theta})$ as: 
\begin{align}
\Prob(\bs{\theta}) = \frac{1}{Z} \prod_{\alpha \in \Index } \psi_{\alpha}(\{\bs{\theta}\}_{\alpha})
\label{eq:prob}
\end{align}
where $\bs{\theta}_{\alpha}$ is a subset of variables in $\bs{\theta}$ that the function $\psi_{\alpha}$ depends on. 
$\psi_{\alpha}: \{\bs{\theta}\}_{\alpha} \rightarrow \mathcal{R}^{+}$ is a potential function, which maps every assignment of variables in $\{\bs{\theta}\}_{\alpha}$ to a non-negative real value. 
$\Index$ is an index set.
$Z$ is a normalization constant, which ensures that the probability adds up to one: $Z = \sum_{{\theta} \in \{0,1\}^{|E|}}\prod_{\alpha \in \Index } \psi_{\alpha}(\{\theta\}_{\alpha})$. In this paper, we will also use $\widetilde{\Prob}(\theta)$ to represent the unnormalized probability density:
\begin{align}
\nonumber\widetilde{\Prob}(\theta) = \prod_{\alpha \in \Index } \psi_{\alpha}(\{\bs{\theta}\}_{\alpha} = \{{\theta}\}_{\alpha} )
\end{align}

A potential function $\psi_{\alpha}(\{\bs{\theta}\}_{\alpha})$ defines the correlation between all variables in the subset $\{\bs{\theta}\}_{\alpha}$. 
As a simple example, suppose $\{\bs{\theta}\}_{\alpha}$ contains the subset of variables for all network edges within a small spatial region. 
All roads in set $\{\bs{\theta}\}_{\alpha}$ may be either all present or all absent, depending on whether a regional event (e.g., an earthquake) happens or not.
In this case, $\psi_{\alpha}(\{\bs{\theta}\}_{\alpha})$ will take high values when all variables are $1$ or $0$ simultaneously. 
The structure of the MRF or the set $\Index$ can be built from domain knowledge and potential functions can be learned from real-world data.
The focus of this paper is not on how to construct the MRF but is how we solve problem in Equation~\ref{eq:multi_source} in general when $Pr(\bs{\theta})$ is given in the form shown in Equation~\ref{eq:prob}.

\begin{thm}
For given policy $\pi$, evaluating the objective of problem~(\ref{eq:multi_source}) is $\#P$ complete.
\end{thm}
\begin{proof}
  This is a consequence of the $\#P$-completeness of the graph reliability problem \cite{Dyer2006StochasticProgramming}, which is to compute the probability that two given vertices is connected if the edges are present or absent independently. Our problem simplifies to the graph reliability problem  when all potential functions are singular. 
\end{proof}


\vspace{5pt}
\noindent \B{Other Applications~} The stochastic network design problem~(\ref{eq:max_exp}) can be used to model other network optimization problems in a more realisitc way by allowing correlation between random events. Here are two examples.

Barrier removal problem~\cite{wu2014rounded,wu2017robust} is an important problem in computational sustainability where the goal is to maximize the expected amount of habitat that one randomly selected group of fish can reach. 
Previous problem formulation assumes that the probability that a group of fish is capable of passing a barrier is independent of the probabilities of passing other barriers.
However, in practice, the probability for a group of fish to pass a barrier mainly depends on the structure of the barrier and the capability of the fish. Therefore, it is more realistic that a fish is able to pass multiple structurally similar barriers or pass none of them. 
Our algorithm introduced later can be applied to solving the barrier removal problem with correlated passage probabilities as well.

Influence maximization problem~\cite{Kempe03} is a well-known problem in social computing, in which people who are infected by an influence will continue to infect their friend stochastically.
The independent cascade model assumes that each person will infect another person with some probability independent from  infections from other sources.
Nevertheless, the infections coming from multiple friends can be correlated, since friends from the same social group tend to post messages of similar content.

\section{Our Method}

\begin{figure}[t]
\setcounter{tempcounter}{\value{equation}}
\setcounter{equation}{0}
\begin{framed}
{\scriptsize
\begin{align}
  &\label{mip:obj}\max_{y}~~~~ \frac{1}{N}\sum^N_{i=1} \sum_{v \in V} w(v)\ z^i_v \\
& \nonumber \text{subject to} \\
& \label{mip:con1}\sum_{(r, v) \in E} \!\!\! x^i_{rv}  - \!\!\!\sum_{(u, r) \in E} \!\!\! x^i_{ur} = \left \{
\begin{array}{l l}
 \sum_{k\neq s} z^i_k & \!\! \text{if }  r = s\\
 - z^{i}_r & \!\! \text{if }  r\neq s
\end{array}
\right. \!\!\!\ \forall r \in V, ~i=\! 1:\! N\\
&\label{mip:con2} 0\leq x^i_{e} \leq (|V| - 1) \left(\sum_{E_s: e\in E_s} y_{E_s} + \theta^i_e \right)~~~~~~~~\forall E_s \in \Es, ~i=\! 1:\! N \\
&\label{mip:con4} 0\leq x^i_{e} \leq |V|-1 ~~~~~i=1:N, ~\forall e\in E\\
&\label{mip:con3} \sum_{E_s\in \Es} c_s y_{E_s} \leq \Bu \\
&\label{mip:con5} y_{E_s} \in \{0, 1\} ~~~\forall E_s \in \Es ~~~~~~ z^i_v \in [0, 1] ~~~\forall i=1:N, \forall v\in V
\end{align}
}
\end{framed}
\caption{\small A compact MIP encoding of problem~(\ref{eq:saa}).}
\label{fig:MIP}
\setcounter{equation}{\value{tempcounter}}
\end{figure}

In this chapter, we introduce the algorithm to solve the single source network design problems (\ref{eq:max_exp}) with 
$\Prob(\bs{\theta})$ defined by (\ref{eq:prob}).
The algorithm can be extended to solve the multiple source network problem (\ref{eq:multi_source}) with a little modification to the MIP formulation. 
We use the Sample Average Approximation~\cite{SAA02} (SAA) method to solve problem~(\ref{eq:max_exp}).
The basic idea is to draw $N$ samples $\{\theta^1, ...., \theta^N\}$ of $\bs{\theta}$ from the distribution $\Prob(\bs{\theta})$ and convert the stochastic optimization problem~(\ref{eq:max_exp}) into the following deterministic optimization problem: 
\begin{align}
{\small \max_{\pi \subseteq \Es} \frac{1}{N} \sum^N_{i=1} \sum_{v \in V}\I(s \leadsto v | \theta^i, \pi) w(v) ~~~~s.t. ~~~\sum_{E_s\in \pi} c_s \leq \Bu }
\label{eq:saa}
\end{align}

\begin{thm}
Let $\pi^*$ be the optimal policy of problem~(\ref{eq:saa}) with $N$ samples and $\pi^{OPT}$ be the optimal policy of problem~(\ref{eq:max_exp}). The number of samples for $\pi^*$ to converge to $\pi^{OPT}$ is $\Theta(|\Es|)$~\cite{SAA02}.
\end{thm}

Problem~(\ref{eq:saa}) can be formulated as a MIP. 
Using the max-flow encoding from one source to multiple targets, we introduce a more compact MIP formulation than the one in paper~\cite{Wu2015} which uses $O(|V|)$ times more variables ($|V|$ is the number of nodes). 
This new compact formulation is shown in Fig.\,\ref{fig:MIP}. 
The basic idea is as follows. 
For each sample $\theta^i$, and for each node $v$, we define a binary variable
$z^i_v$, which represents whether node $v$ is reachable from $s$ $(=1)$ in the $i^{th}$ sample or not $(=0)$.
Variables $z^i_v$ can be relaxed to be continous variable as shown in Fig.\,\ref{fig:MIP} because the value of $z^i_v$ is either $0$ or $1$ in the optimal solution.
The objective function in Equation~\ref{eq:saa} can be written as:
\begin{align}
\nonumber\frac{1}{N}\sum^N_{i=1} \sum_{v \in V} w(v)\ z^i_v
\end{align}
We then defines an indicator variable $y_{E_s}$ for the action to protect all edges in $E_s$. Protecting edges in $E_s$ costs $c_s$, therefore we have the budget constraint (\ref{mip:con3}) in Fig.\,\ref{fig:MIP}. 

Next, we enforce flow constraints (\ref{mip:con1},\ref{mip:con2},\ref{mip:con4} in Figure~\ref{fig:MIP}) to guarantee that $z_k^i=1$ if and only if node $k$ is reachable from $s$ in the  $i^{th}$ sample. The idea is to inject $\sum_{k\neq s} z^i_k$ unit of flow into the source $s$, and to pull out $z^i_k$ unit of flow ($z^i_k$ is either 1 or 0, depending on whether node $k$ is reachable from $s$ in the $i^{th}$ sample or not) at all nodes other than the source $s$.
Here we use $x_e^i$ as flow variables to represent the flow on edge $e$ in the $i$-th sample. 
Constraint (\ref{mip:con1}) is used to enforce a flow of this kind.
If a node $k$ is reachable from $s$ $(z^i_k = 1)$, one unit of flow will be consumed by node $k$. 
If $k$ is not reachable from $s$ $(z^i_k = 0)$, $k$ will not consume any flow.
Constraint (\ref{mip:con4}) restricts the amount of flow on each edge between 0 and $|V|-1$, which must be true since the total amount of flow is $\sum_{k\neq s} z^i_k$ and we have $\sum_{k\neq s} z^i_k \leq |V|-1$.
Constraint (\ref{mip:con2}) ensures that no flow is allowed on edge $e$ unless $e$ is present in the $i$-th sample ($\theta_e=1$) or $e$ is protected by at least one action $E_s$ that covers $e$ and the corresponding $y_{E_s} = 1$.

SAA is sensitive to the number of samples $N$ used to approximate the objective function. Although in general a large $N$ is preferred in order to reduce the variance, the number of variables in the MIP encoding in Figure~\ref{fig:MIP} scales linearly with respect to $N$, which prevents us from using a large number of samples (e.g., millions of samples) in practice.
Under the independent edge sampling case, despite theoretical results \cite{SAA02} claim that $O(|\mathcal{E}|)$ number of samples are required to obtain probabilistic guarantees, in practice, people found that a small number of samples (e.g., $10$) are sufficient to find a good policy~\cite{Sheldon10}.
We call this small set of samples as \emph{good samples} or \emph{well-representative samples}.
That is, the number of samples is relatively small but these samples can be used to find good policies.
To make SAA work, it is crucial to \emph{obtain a small-size ($N$ small) but well-representative set of samples}, so that not only the quality of the sample approximation is good, but also the size of the sample set is small enough for MIP to scale. 
It is unclear and potentially more difficult to obtain good samples when edges are correlated. 
In the following sections, we introduce two candidate methods to attempt to obtain these so-called good samples $\{\theta^1,...,\theta^N\}$.


\subsection{Gibbs Sampling}

Gibbs sampling is a special case of the Metropolis-Hastings algorithm to obtain a sequence of samples which approximates a given distribution $\Prob(\bs{\theta})$. The high-level idea is to build a Markov chain whose transition probability is 
\begin{small}
\begin{align}
  \Prob(\theta_i | \theta_1, .., \theta_{i-1}, \theta_{i+1}, .., \theta_{|E|}) = \frac{\widetilde{\Prob}(\theta_1, \dots, \theta_{|E|})}{\widetilde{\Prob}(\theta_1, .., \theta_{i-1}, \theta_{i+1}, .., \theta_{|E|})}.\nonumber
\end{align}
\end{small}
It has been proved that the stationary distribution of the Markov chain converges to the target probability distribution $\Prob(\theta)$.
Notice that we only need the unormalized density $\widetilde{\Prob}(.)$, rather than $\Prob(.)$ to compute $\Prob(\theta_i | \theta_1, .., \theta_{i-1}, \theta_{i+1}, .., \theta_{|E|})$. This is the crucial fact that makes the transition of Gibbs sampling efficient in practice.

\subsection{XOR Sampling}

Sampling with XOR constraints \cite{Gomes2006Sampling,Gomes06XORCounting,ermon2013embed,Chakraborty2013} was developed, thanks to the
rapid progress in solving NP-complete problems in the last decade.
Despite the intractability in the worst case, modern MIP and Satisfiability Testing (SAT) solvers are capable
of exploiting the problem structure, and can scale up to
millions of variables \cite{Knot2010,Sontag08,Riedel2008,Poon2006mcsat,Gogate2011SampleSearch}.
XOR sampling transforms the problem of sampling, which is an
\#P-complete problem, into solving satisfiability problems with
additional XOR constraints.

With XOR sampling, the resulting sampling distribution can be proved
to be arbitrarly close to uniform in the unweighted case \cite{Gomes2006Sampling,Chakraborty2013}, and
to be within a  constant factor of the true
probability distribution in the weighted case \cite{ermon2013embed}. 
This is a significant improvement compared with MCMC based techniques,
which often require an exponential amount of samples to reach its
stationary distribution on many problems with intricate combinatorial structure.

XOR sampling is a special fit to our SAA framework. The rationale
behind XOR sampling is to trade computational efficiency with provable
near-optimal quality in the samples obtained.
XOR sampling is more expensive than Gibbs sampling because it has to solve combinatorial problems at each sampling step. 
The recent progress in
constraint programming makes it possible to solve large-scale
NP-complete problems efficiently, therefore the benefit often 
outweights the cost.
This is especially true in our SAA framework.  First, the MIP encoding 
for SAA prevents us from optimizing over millions of examples
simultaneously. Therefore, the samples we use in the optimization has
to be small, but representative enough for SAA to balance over all likely
scenarios.  Second, since the samples are obtained prior to the SAA
optimization, we can afford to spend some time on obtaining good
samples. We can even run several sampling programs in parallel.
In experiments, we found that the time spent at the sampling step was
negligible compared with solving the MIP.

\input{xor}

\section{Related Work}
Our work is related to an active research field on reliability of physical communication networks under natural disasters and malicious coordinated attacks~\cite{rahnamay2011modeling,neumayer2010network}.
In those papers, probabilistic models are built based on \emph{spatial point processes} to evaluate the reliability of networks with geographically correlated failures. 
However, the problem of optimizing the reliability of networks by taking protection actions is not discussed and the proposed models are too complex to be used to define solvable decision-making problems for network protection.

\citep{Sheldon10} formulate the problem of maximizing the spread of cascades as stochastic network design and develop an algorithm combining SAA and MIP.
In their model, all variables are assumed to be independently distributed so samples of each variable can be drawn independently. 
Their MIP works only when the network is a DAG and has only one source while our MIP works for general directed graphs and an arbitrary number of sources.
\citep{wu2014stochastic} solve the river network design problem with a single source and multiple sources by rounded dynamic programming algorithms \cite{wu2014rounded,wu2014stochastic}.
However, their algorithms only work for tree-structured networks, and they also assume that stochastic events are all independent.
Recently, several fast approximate algorithms based on the primal-dual technique have been developed to solve several deterministic network design problems similar to our problem~(\ref{eq:saa})~\cite{Wu2015,xiaojian2016}, but they do not guarantee to produce the optimal solutions.
It is an interesting future work to see how we can combine their algorithms and XOR sampling techniques to tradeoff the optimality with efficiency on large-scale networks.
\citep{Xue2016SCR} use an optimization algorithm by adding XOR constraints to solve a related landscape connectivity problem. Their optimization problem takes into account of spatial capture re-capture model, which is in a different setting from ours. More importantly, we use XOR sampling within the SAA framework, which is different from their approach. 
\nocite{Xue2016MarginalMAP}0

\section{Experiments}
We test our algorithm on a real-world problem, the so-called Flood Preparation problem for the emergency medical services (EMS), on road networks introduced in paper~\cite{xiaojian2016}. 
Edges of the graph represent road segments.
Nodes represent either the intersections of roads or locations where EMS are requested through phone calls, which are sampled based on the patient population density.
Some edges or road segments are above roadway stream crossing structures, such as culverts, which are vulnerable to floods and earthquakes.
The failure of a crossing structure will disrupt the road segment above it and cause the road to become impassable.
In our study area, only a subset of edges are associated with crossings, each of which defines a binary variable $\bs{\theta}_e$.
The value of $\bs{\theta}_e$ indicates whether the crossing under $e$ fails ($=0$) or not ($=1$).
A protection action can replace a crossing with a sturdy type such that it will not fail during floods.
The goal is to decide which crossings (or which edges out of the subset of edges above crossings) to protect beforehand so patients remain connected to ambulance centers even during floods, that is, we will decide for each edge above a crossing whether to protect it or not.
To measure the connectivity during floods, we use the objective of the multiple source case introduced in section 2.
Each source represents an ambulance center.
Since each center usually has a limited number of ambulances available (sometimes only one ambulance available), a location reachable from multiple centers has a lower risk of no ambulance responding during floods than a location reachable from only one center.
The problem can also be treated as a single source problem in which it is sufficient to connect a node to only one center or being able to connect to more than one center doesn't further increase the objective.
While our methodology applies to both cases, we only show the results for the multiple source case.

We use a similar way as in paper~\cite{neumayer2010network,rahnamay2011modeling} to synthetically generate failure correlation between these crossing structures.
The idea is as follows.
Each crossing may fail with a small chance independently and multiple spatially nearby crossings tend to fail concurrently if affected by a natural disaster.
To generate disaster locations, instead of using spatial point processes as in paper~\cite{neumayer2010network} due to the lack of landscape features, we uniformly generate a small number of disaster centers and define their effect regions with randomly generated radiuses.
All variables within a region are correlated with different levels of strengths (strong or weak).

\begin{figure}[t]
\subfloat[{\small Strong, $20$ sources}]{\includegraphics[height=1.3in]{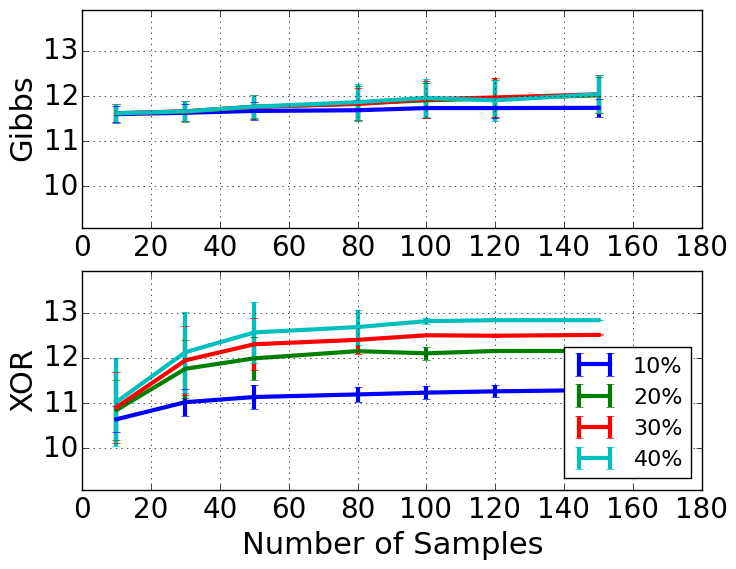}}
\subfloat[{\small Weak, $20$ sources}]{\includegraphics[height=1.3in]{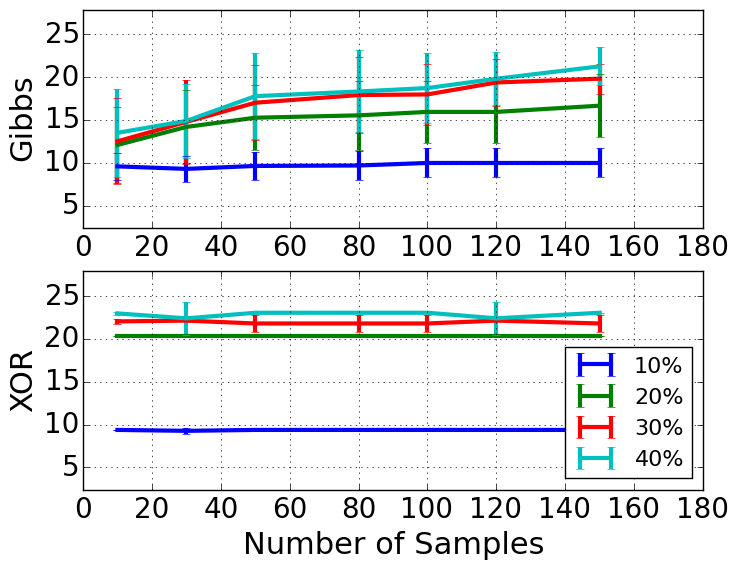}} \\
\subfloat[{\small Weak, $2$ sources}]{\includegraphics[height=1.3in]{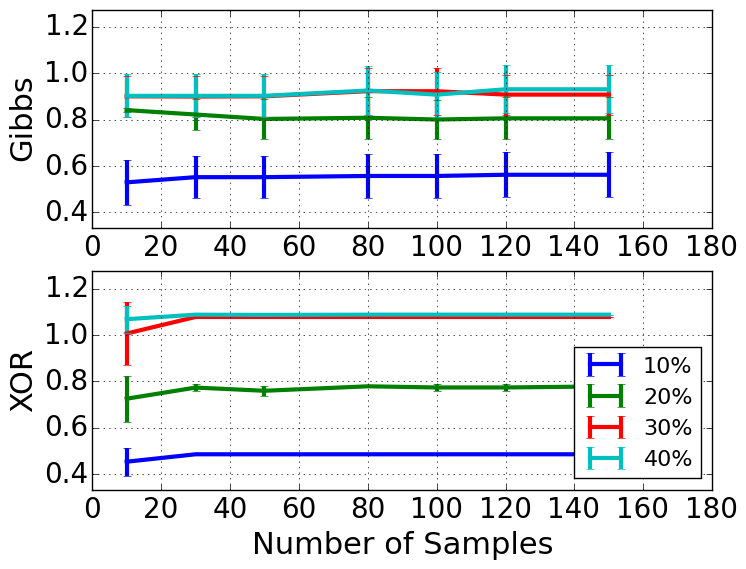}} 
\subfloat[{\small Weak, $10$ sources}]{\includegraphics[height=1.3in]{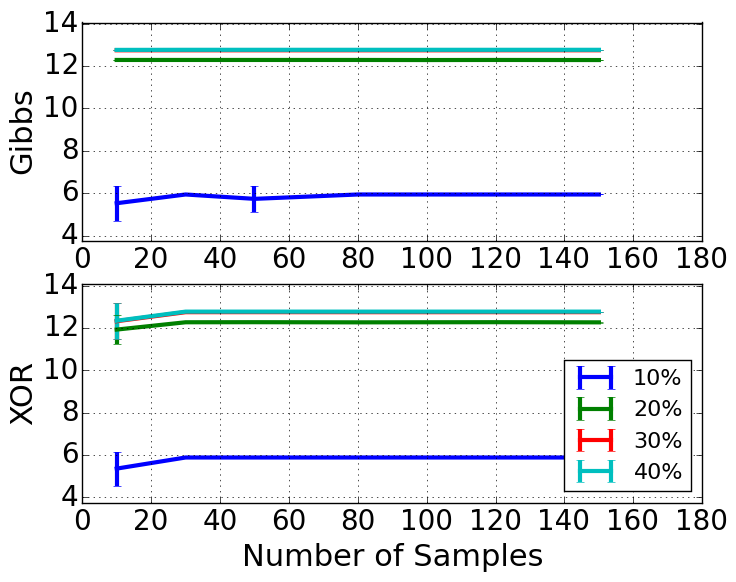}} 
\caption{SAA with the XOR sampler produces better and more stable policies (e.g., larger mean and smaller error bars) than SAA with the Gibbs sampler on small networks. Y axis is the average of $10$ policy values ($\times 10^3$). Variances are illustrated by vertical bars that indicate the interval [mean - std, mean + std] (std: standard deviation). ``Strong'': strong correlation. ``Weak'': weak correlation. Legend: budget sizes as the percentage of the total costs of protecting all edges. }
\label{fig:small_network}
\end{figure}

\vspace{5pt}
\noindent\B{Small Networks~} We test our algorithm on a small network of $502$ edges and $20$ crossings extracted from the road networks, which consists of $20$ binary random variables. 
$20$ is small enough to compute the exact objective value of Equation (\ref{eq:multi_source}) for any given policies so we can compare different policies exactly.
We test on a different number of sources.
For each problem instance, we vary the sample size from $10$ to $180$ and budget size from $10\%$ to $40\%$ of the total cost to protect all crossings.
For sample size $N$, we produce $10$ policies, each computed with $N$  samples independently drawn from the Gibbs and XOR sampler. 
The objective value of each policy is computed with an exact solver.
The mean and variance over $10$ policy values are reported. 
The results of four ``representative'' instances are shown in Fig.\,\ref{fig:small_network}, and the results of other experimental instances show similar trends.
In most of the sample sizes and testing instances, the mean values produced by SAA with the XOR sampler is as good or better than those obtained by SAA with the Gibbs sampler. 
Moreover, the policy produced by the XOR sampler is more stable in the sense that the variance of policy values is smaller than the variance produced by the Gibbs sampler, as shown in Fig.\,\ref{fig:small_network}(a-c). 
The results by the Gibbs sampler are comparable to the results by the XOR sampler in some instances.
In Fig.\,\ref{fig:small_network}(d), both methods produced almost identical means and variances.
In Fig.\,\ref{fig:small_network}(a), the XOR sampler gives smaller means with $10$ samples but quickly outperforms the Gibbs sampler with more samples.
In Fig.\,\ref{fig:small_network}(a)(c), for the $10\%$ budget cases, the Gibbs sampler produces larger means but appears very unstable.
In most of these results, as the number of samples increases, the variance by the XOR sampler quickly reduces to $0$ while the variance by the Gibbs sampler remains large (e.g., in Fig.\,\ref{fig:small_network}(b) and (c)).

\begin{figure}[t]
\centering
\includegraphics[height=1.6in]{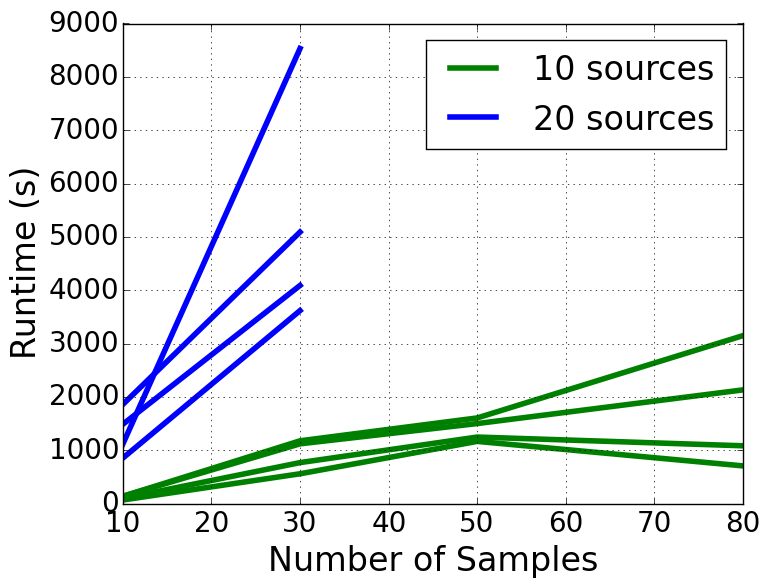}
\caption{Runtime of MIP versus sample sizes. Lines in the same color represent different budget sizes. The runtime value is the average of $30$ experiments, each with independently drawn samples.}
\label{fig:runtime}
\end{figure}

\vspace{5pt}
\noindent\B{Large Networks~} We test on a large road network of $1,562$ edges and $81$ binary random variables where the exact policy value is not available.
Fig.\,\ref{fig:runtime} shows that the MIP quickly becomes unscalable as the sample size and the number of sources increase.
CPLEX takes around $2.5$ hours to solve the MIP of $30$ samples and $20$ sources and cannot finish within $4$ hours for $50$ samples and $20$ sources.
Therefore, to work with $20$ or more sources or networks of larger sizes, it is important that good policies can be obtained by a small number of samples (e.g., $30$ samples for $20$ sources).
For $10$ source case, it is interesting that the runtime decreases a little bit after a certain number of samples for some budget sizes.

Also, for both samplers, we do experiments to find out the number of samples that are sufficient for the sample average method to obtain a reasonably good estimation of policy value.
To test whether a sequence of values converges, we borrow the idea from Cauchy sequence test to see if the difference of two subsequent values converges to zero. 
The results are shown in Fig.\,\ref{fig:eval}.
We create a sequence of sample sizes $\{N_1, N_2,...\}$ varying from $10$ to $5000$.
In Fig.\,\ref{fig:eval}, the value reported on y-axis at $x=N_i$ equals the difference between the average of $N_i$ policy values and the average of $N_{i-1}$ policy values, each policy value estimated by an independent drawn sample of $\bs{\theta}$. 
If the difference goes to zero, it implies the sequence converges. 
As one can see from the figure, the value difference for both the XOR and the Gibbs sampling narrows down to a small interval while the value difference for the XOR sampler converges to a much smaller range near $0$ and much faster than for the Gibbs sampler. 
The results also imply that the empirical mean of $5000$ samples can produce good estimations of the policy value.

\begin{figure}[t]
\subfloat[The Gibbs sampler]{\includegraphics[height=1.3in]{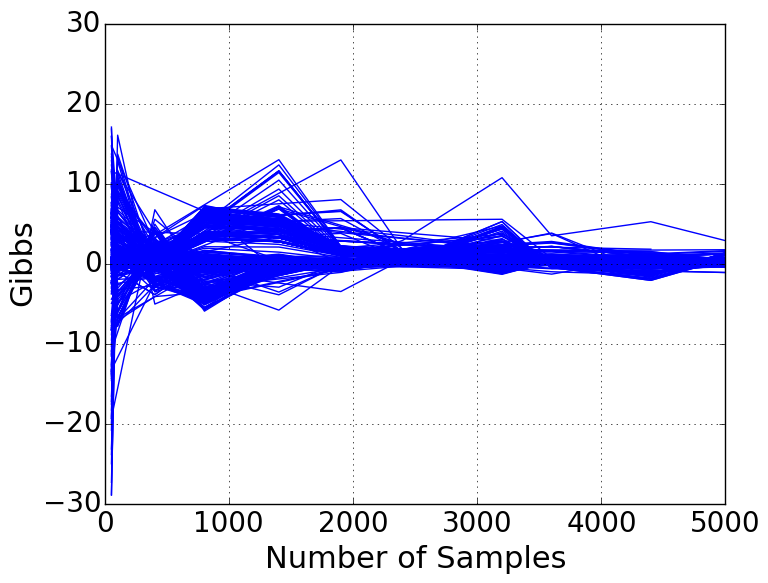}}
\subfloat[The XOR sampler]{\includegraphics[height=1.3in]{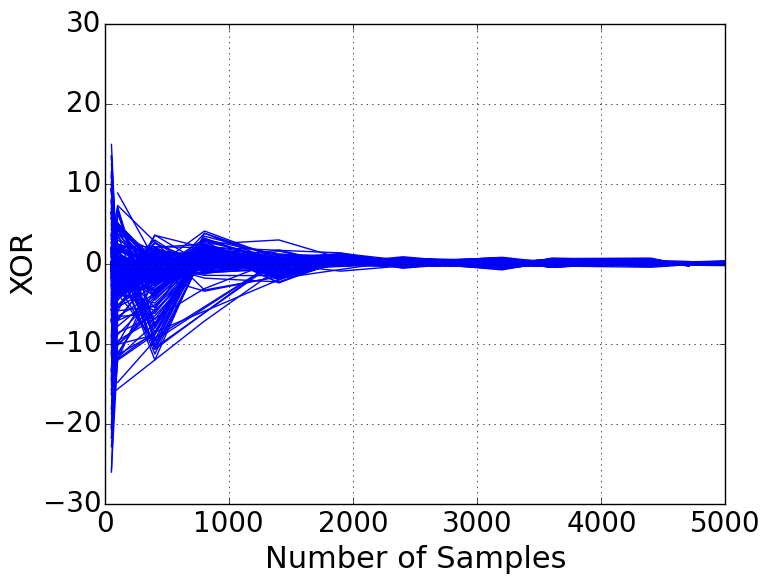}}
\caption{The effect of sample sizes on estimation of policy values. Each line corresponds to a policy.  For one policy, the value of Y axis is the difference of two subsequent estimated policy values. For both the XOR and the Gibbs samplers, the difference reduces to a small value with $5,000$ samples. The XOR sampler converges better and faster comparing to the Gibbs sampler. }
\label{fig:eval}
\end{figure}

At last, we compare the solution qualities of our algorithm on large networks. Two representative results are shown in Fig.\,\ref{fig:large_net}. 
Again, the results of instances that are not shown in the figure have similar trends. 
Since we can not calculate the exact value of a policy (e.g., an expectation) on the large network, we estimate the expectation using $5000$ empirical values of the policy, each calculated by an independent sample.
As demonstrated, $5000$ samples are sufficient to obtain a good estimation.
In the figure, the XOR sampler outperforms the Gibbs sampler with better and stable policy values, in a similar way as for the small networks. 
\begin{figure}[t]
\centering
\subfloat[$1$ source]{\includegraphics[height=1.3in]{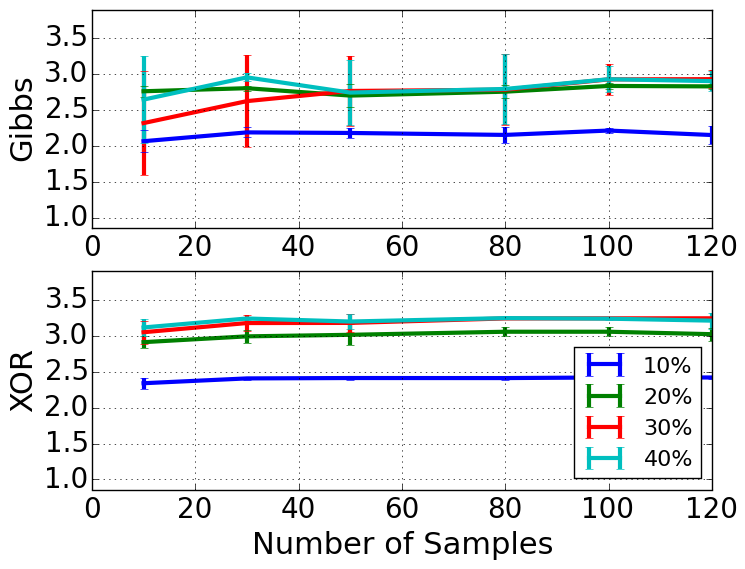}}
\subfloat[$10$ sources]{\includegraphics[height=1.3in]{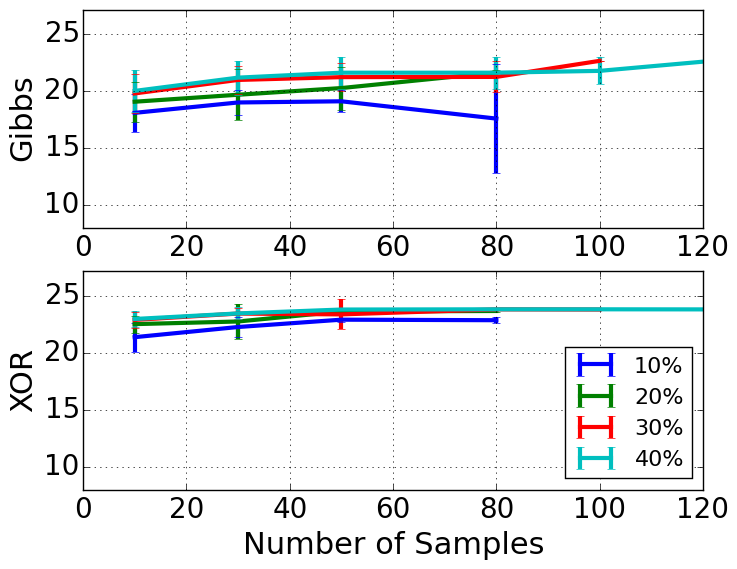}}
\caption{SAA with the XOR sampler produces better and more stable policies than SAA with the Gibbs sampler on large networks with $1$ and $10$ sources. Axes have the same meanings as in Fig.\,\ref{fig:small_network}.}
\label{fig:large_net}
\end{figure}
In conclusion, the XOR sampler produces solutions with higher qualities and is less sensitive to sample variance.

\section{Conclusion}
Previous approaches for solving stochastic network design problems rely on a restricted assumption that the states of edges are independently distributed. An SAA based algorithm can produce high-quality solutions by taking advantage of this restricted assumption, although they only scaling up to a small number of samples. 
In this paper, we relax the independence assumption and propose a more general problem framework by allowing correlated stochastic events to affect the states of edges in networks.  
Novel algorithms have been developed based on the sample average method and various samplers. We show that our algorithm can produce high-quality solutions with a small number of samples if these samples are obtained from a XOR sampling technique.
Testing on a real-world flood preparation problem for EMS, empirical results show that SAA with the XOR sampler produces better and more stable solutions compared with SAA with a Gibbs sampler.

\section*{Acknowledgements}
This research was supported by the National Science Foundation
under Grant No. CNS-0832782, CNF-1522054, CNS-1059284, ARO grant W911-NF-14-1-0498
and the Atkinson Center for a Sustainable Future.

\bibliographystyle{named}
\bibliography{ijcai17}

\end{document}

%% file: xor.tex
\vspace{5pt}
\noindent {\bf General Idea~} To illustrate the high-level idea of XOR
sampling, let us first consider the unweighted case where the
unnormalized density $\widetilde{\Prob}(\theta)$ is either 1 or
0. Suppose there are $2^l$ different $\theta$ assignments which make
$\widetilde{\Prob}(\theta)=1$. We call all these assignments
``satisfying assignments''. An ideal sampler draws samples from the
set of satisfying assignments $\Delta=\{\theta:
\widetilde{\Prob}(\theta)=1\}$ uniformly at random; i.e., each member
in $\Delta$ has $2^{-l}$ probability to be selected.

At a high level, XOR sampling obtains near-uniform samples by finding
one satisfying assignment subject to additional randomly generated XOR
constraints.
To be specific, we keep adding XOR constraints to the original problem
of finding satisfying assignments. We can prove that in expectation,
each newly added XOR constraint rules out about half of the satisfying
assignments at random.
Therefore, if we start with $2^l$ satisfying assignments in $\Delta$,
after adding $l$ XOR constraints, we will be left with only one
satisfying assignment in expectation, by which time we return this
assignment as our first sample.
%
Because we can prove that the assignments are ruled out at random in each
step, we can guarantee that the last assignment left must be a
randomly chosen one from the original set $\Delta$.
%
Please see \cite{Gomes2006Sampling} for details.


For the weighted case, our goal is to guarantee that the probability
of sampling one $\theta$ is propotional to the unnormalized density
$\widetilde{\Prob}(\theta)$. We implement a horizontal slice technique
to transform a weighted problem into an unweighted one, taking a
2-approximation. For the easiness of illustration, consider the simple
case where $P_0 = \min_{\theta}
\widetilde{\Prob}(\theta)$, $P_k = \max_{\theta}
\widetilde{\Prob}(\theta)$ and $P_k = 2^k P_0$.
Let $\delta=(\delta_0, \ldots, \delta_{k-1}) \in \{0,1\}^k$ be a binary vector of length $k$. We sample $(\theta, \delta)$ uniformly at random from the following set $\Delta_w$ using the afromentioned unweighted XOR sampling:
$$\Delta_w= \{(\theta, \delta): \widetilde{\Prob}(\theta) \leq 2^{i+1} P_0 \Rightarrow \delta_i = 0 \}.$$
If we sample $(\theta, \delta)$ uniformly at random from $\Delta_w$ and  then only return $\theta$, it can be proved that the probability of sampling $\theta$ is propotional to $P_0 2^{i-1}$ when $\widetilde{\Prob}(\theta)$ is sandwiched between $P_0 2^{i-1}$ and $P_0 2^i$. 
We refer the readers to \cite{ermon2013embed} for details. 